\documentclass[orivec,runningheads]{llncs}
\usepackage[T1]{fontenc}
\usepackage{graphicx}
\usepackage{amsfonts}
\usepackage{amsmath}
\usepackage{mathtools}
\usepackage{hyperref}
\usepackage{enumerate}
\usepackage{enumitem}
\usepackage{tikz}
\usepackage{booktabs}
\usepackage[ruled,vlined,linesnumbered]{algorithm2e}
\usepackage{stmaryrd}
\usepackage{attachfile2}
\usepackage{color}
\usepackage{xcolor}
\usepackage{soul}

\usepackage{{makecell}}
\setcellgapes{3pt}\makegapedcells

\usepackage{array,collcell}
\newcommand\AddLabel[1]{%
  \refstepcounter{equation}
  (\theequation)
  \label{#1}
}
\newcolumntype{M}{>{\hfil$\displaystyle}X<{$\hfil}} 
\newcolumntype{L}{>{\collectcell\AddLabel}r<{\endcollectcell}}

\newcommand{\constraint}[1]{\mbox{\sc{#1}}}

\newcommand{\partition}{\constraint{Partition}}
\newcommand{\group}{\constraint{BinSeq}}

\newcommand{\avalue}{j}

\newcommand{\ceil}[1]{\ensuremath{\left\lceil #1 \right\rceil}}
\newcommand{\floor}[1]{\ensuremath{\left\lfloor #1 \right\rfloor}}

\newcommand{\nval}{\ensuremath{P}}
\newcommand{\pn}{n}
\newcommand{\pmin}{\underline{M}}
\newcommand{\pmax}{\overline{M}}

\newcommand{\prange}{\underline{\overline{M}}}

\newcommand{\prmin}{R}
\newcommand{\psumsquares}{S}

\newcommand{\sm}{\mathit{SM}}
\newcommand{\rr}{\mathit{RR}}
\newcommand{\mmid}{\mathit{MID}}
\newcommand{\smin}{\mathit{SMIN}}
\newcommand{\omin}{o_{\mathit{min}}}
\newcommand{\omax}{o_{\mathit{max}}}

\newcommand{\gn}{\mathit{n}}
\newcommand{\gng}{\mathit{G}}
\newcommand{\gnv}{\mathit{N_1}}
\newcommand{\gmin}{\underline{G}}
\newcommand{\gmax}{\overline{G}}
\newcommand{\dmin}{\underline{D}}
\newcommand{\dmax}{\overline{D}}
\newcommand{\grange}{\overline{\underline{G}}}
\newcommand{\drange}{\overline{\underline{D}}}
\newcommand{\gsumsquares}{\mathit{GS}}
\newcommand{\dsumsquares}{\mathit{DS}}

\newcommand{\sequence}{\mathcal{X}}

\newcommand{\Selection}{\mbox{Selection}}
\newcommand{\Enumerate}{\mbox{Enumerate}}
\newcommand{\ComputeAllSolutions}{\mbox{ComputeAllSolutions}}
\newcommand{\Select}{\mbox{Select}}
\newcommand{\EnumerateAllSolutions}{\mbox{EnumerateAllSolutions}}
\newcommand{\Labeling}{\mbox{Labeling}}
\newcommand{\SelectOne}{\mbox{SelectOne}}
\newcommand{\Dicho}{\mbox{Dicho}}
\newcommand{\candidates}{\mathit{Bounds}}
\newcommand{\ctr}{\mathit{Ctr}}
\newcommand{\charvars}{\mathit{Feat\!Vars}}
\newcommand{\vars}{\sequence}
\newcommand{\sols}{\mathit{Sols}}
\newcommand{\bound}{\mathit{Bound}}
\newcommand{\isol}{\mathit{ISol}}
\newcommand{\finished}{\mathit{Finished}}
\newcommand{\sol}{\mathit{Sol}}
\newcommand{\nback}{\mathit{NBack}}
\newcommand{\prevselected}{\mathit{PrevBound}}
\newcommand{\newselected}{\mathit{Selected}'}
\newcommand{\newcandidates}{\mathit{Bounds}'}
\newcommand{\restselected}{\mathit{RestSelected}}
\newcommand{\toplevel}{\mathit{Top}}
\newcommand{\len}{\mathit{Len}}
\newcommand{\midd}{\mathit{Mid}}
\newcommand{\prefix}{\mathit{Prefix}}
\newcommand{\suffix}{\mathit{Suffix}}
\newcommand{\missing}{\mathit{MissingBound}}
\newcommand{\sortedsols}{\mathit{SortedSols}}
\newcommand{\allsols}{\mathit{AllSols}}
\newcommand{\jsol}{\mathit{JSol}}
\newcommand{\back}{\mathit{Back}}

\newcommand{\mycolor}[1]{\color{#1}}
\renewcommand{\mycolor}[1]{}
\newcommand{\jomodif}[1]{{\mycolor{red}#1}}

\newcommand{\cgdeleted}[1]{\st{#1}} 
\renewcommand{\cgdeleted}[1]{}    

\DeclareRobustCommand{\hlpink}[1]{{\sethlcolor{pink}\hl{#1}}}

\newcommand{\rebuttal}[2]{\noindent\hlpink{\textbf{Reviewer: } #1} \hl{\textbf{Our response: } #2}}
\renewcommand{\rebuttal}[2]{} 

\begin{document}

\title{Incremental Selection of\\ Most-Filtering Conjectures and\\ Proofs of the Selected Conjectures}

\titlerunning{Acquiring and Selecting Implied Constraints}
\author{J.~Cheukam Ngouonou\textsuperscript{\rm 1,2,4}, R.~Gindullin\textsuperscript{\rm 1,2}, C.-G. Quimper\textsuperscript{\rm 4},\\ N. Beldiceanu\textsuperscript{\rm 1,2} and R. Douence\textsuperscript{\rm 1,2,3}}
\authorrunning{Cheukam Ngouonou et al.}
\institute{    \textsuperscript{\rm 1}IMT Atlantique, Nantes, France
    ~~\textsuperscript{\rm 2}LS2N, Nantes, France\\
    \textsuperscript{\rm 3}INRIA, Nantes, France
    ~~\textsuperscript{\rm 4}Université Laval, Qu{\'e}bec, Canada\\
\email{ramiz.gindullin@it.uu.se, nicolas.beldiceanu@imt-atlantique.fr, jovial.cheukam-ngouonou.1@ulaval.ca, Remi.Douence@imt-atlantique.fr, Claude-Guy.Quimper@ift.ulaval.ca}}
\maketitle            

\begin{abstract}
We present an improved incremental selection algorithm of the selection algorithm presented in~\cite{CheukamNgouonouGQBD25}
and prove all the selected conjectures.
\end{abstract}

\section{Introduction}

In Section~\ref{sec:incremental_selection_alg}, we describe an incremental algorithm for selecting the most-filtering bound conjectures.
This incremental algorithm and the speedup it offers were mentioned in~\cite{CheukamNgouonouGQBD25}, but were not described for space reasons.
In Section~\ref{sec:partition} and in Section~\ref{sec:group} we respectively prove the selected conjectures for the \partition\ and the \group\ constraints that were not proved in~\cite{CheukamNgouonouGQBD25}.

\section{An Incremental Selection Algorithm}\label{sec:incremental_selection_alg}

We present an incremental version of the selection algorithm described in~\cite{CheukamNgouonouGQBD25}.
Unlike our original algorithms, we do not post the candidate bound constraints from scratch during each step of the dichotomic search;
nor do we scan the set of solutions from scratch when looking for the next candidate bound constraint to select.

In the following, we assume that if a constraint fails while being posted, it will be removed by the solver.
Given the set of constraints already posted $\mathcal{C}$, the function $\Labeling(\charvars,\vars)$ returns a triplet $(\nback,\finished,\sol)$, where $\sol$ is the first solution found that satisfies all the constraints of $\mathcal{C}$ by assigning the variables of $\charvars$ and $\vars$ from left to right, assuming that the variables are fixed by scanning their domains by increasing values:
\begin{itemize}[label=--]
\item
$\nback$ is the number of backtracks to find a solution or prove that there is no solution,
\item
$\finished$ is set to \textsc{true} if no solution could be found, and to \textsc{false} otherwise,
\item
$\sol$ is meaningless if no solution was found.
\end{itemize}

The main selection algorithm, Alg.\hspace*{1pt}(\ref{alg::IncSelection}), has the following arguments:
\begin{itemize}[label=--]
\item $\ctr$ the constraint associated with a combinatorial object, e.g. the \partition\ or the \group\ constraints.
\item $\charvars$ the set of feature variables of constraint $\ctr$, e.g. variables $\nval$, $\pmin$, $\pmax$, $\prange$ and $\psumsquares$ for \partition\ or variables $\gnv,\!\gng$, $\gmin$, $\gmax$, $\grange$, $\gsumsquares$, $\dmin$, $\dmax$, $\drange$, $\dsumsquares$ for \group.
\item $\vars$ the array of variables $[X_1,X_2,\ldots,X_\pn]$ of constraints $\ctr$.
\item $\candidates$ the list of candidate bound constraints found by the Bound Seeker for constraint $\ctr$.
\end{itemize}

From a set of candidate bound constraints $\candidates$, Alg.~(\ref{alg::IncSelection}) returns a list of selected constraints.
Line~1 of Alg.~(\ref{alg::IncSelection}) posts the constraint $\ctr$ only once during the whole selection process, calls Alg.~(\ref{alg::IncComputeAllSolutions}) in line~2 to compute all solutions using all candidate bound constraints, and finally selects in line~3 a subset of candidate bound constraints that lead to the computation of each solution of the constraint $\ctr$ without increasing the number of backtracks.
\begin{algorithm}
\caption{$\Selection(\ctr,\charvars,\vars,\candidates)$\label{alg::IncSelection}}
post constraint $\ctr(\charvars,\vars)$\;
$\sols\gets\ComputeAllSolutions(\charvars,\vars,\candidates)$\;
\Return $\Select(\sols,\charvars,\vars,\candidates,\{\})$\;
\end{algorithm}

Alg.~(\ref{alg::IncComputeAllSolutions}) and~(\ref{alg::IncEnumerateAllSolutions}) compute all solutions of the feature variables $\charvars$ wrt constraint $\ctr$ in ascending lexicographic order on the $\charvars$ variables and records each solution with the number of backtracks to find it.
Since we will need to calculate a specific solution later on in the selection process, independently from the other solutions, we proceed as follows:
to obtain the $i$-th solution wrt the lexicographic order of $\charvars$, we compute the smallest lexicographic solution that is strictly greater than the $(i-1)$-th solution already known. In lines 3--5 of Alg.~(\ref{alg::IncEnumerateAllSolutions}) we post the constraint $\charvars >_\mathit{\ell\!ex} \sol$ stating that $\charvars$ is lexicographically strictly greater than $\sol$ before computing the next smallest lexicographic solution.
Finally, in line~3 of Alg.~(\ref{alg::IncComputeAllSolutions}), we sort all solutions by increasing number of backtracks, as the selection process will use this order to reduce the time spent generating solutions.

\begin{algorithm}
\caption{$\ComputeAllSolutions(\charvars,\vars,\candidates)$\label{alg::IncComputeAllSolutions}}
$\forall\,\bound\in\candidates:~\text{post bound constraint}~\bound~\text{on}~\charvars$\;
$\sols\gets\EnumerateAllSolutions(\charvars,\vars)$\;
$\sortedsols\gets\,\text{sort}~\sols~\text{by increasing number of backstracks}$\;
remove all posted bound constraints\;
\Return $\sortedsols$\;
\end{algorithm}

\begin{algorithm}
\caption{$\EnumerateAllSolutions(\charvars,\vars)$\label{alg::IncEnumerateAllSolutions}}
$\isol\gets 0;$ $\sols\gets\emptyset$\;
\While{\textsc{true}}{
    \If{$\isol>0~\land~\text{post constraint}~\charvars >_\mathit{\ell\!ex} \sol~\text{fails}$}{
        \Return $\sols\cup\{(\isol,0,[])\}$\;
    }
    $(\nback,\finished,\sol)\gets\Labeling(\charvars,\vars)$\;
    $\sols\gets\sols\cup\{(\isol,\nback,\sol)\}$\;
    $\textbf{if}~\isol>0~\textbf{then}~\text{remove lexicographic constraint that was posted}$\;
    $\textbf{if}~\finished~\textbf{then return}~\sols~\textbf{else}~\isol\gets\isol+1$\;
}
\end{algorithm}

Alg.~\ref{alg::IncSelect} is a recursive selection algorithm that selects a subset of bound constraints that does not increase the number of backtracks to find each solution.
At each step of the recursion, Alg.~\ref{alg::IncSelect} successively:
\begin{itemize}[label=--]
\item
Post the previously selected bound constraint $\prevselected$ (lines 1--2), as all previously selected bound constraints must be posted when searching for the next bound constraint to select.
Note that unlike the selection algorithm described in~\cite{CheukamNgouonouGQBD25}, each selected bound constraint is only posted once during the entire selection process.
\item
Select the next bound constraint $\newselected$ from the current set of candidate bound constraints $\candidates$, and create the new reduced set of candidate bound constraints $\newcandidates$ (line~3).
\item
If both, we could select a bound constraint and we still have some candidate bound constraints (line~4), we recursively call Alg.~\ref{alg::IncSelect} to select the next bound constraints to keep (line~5).
\end{itemize}

\vspace{-0.5cm}
\begin{algorithm}
\caption{$\Select(\sols,\charvars,\vars,\candidates,\prevselected)$\label{alg::IncSelect}}
\If{$\prevselected\neq\emptyset$}{
    post previous selected bound constraint $\prevselected$ on $\charvars$\;
}
$(\newselected,\newcandidates)\gets\SelectOne({\textsc{true}},\hspace*{-2pt}\sols,\hspace*{-2pt}\sols,\hspace*{-2pt}\charvars,\hspace*{-2pt}\vars,\hspace*{-2pt}\candidates)$\;
\uIf{$\newselected\neq\emptyset\land\newcandidates\neq\emptyset$}{
    $\restselected\gets\Select(\sols,\charvars,\vars,\newcandidates,\newselected)$\;
    \Return $\newselected\cup\restselected$\,\;
}\lElse{
    \Return $\newselected$
}
\end{algorithm}

\vspace{-0.5cm}
Alg.~\ref{alg::IncSelectOne} recursively selects the next bound constraint from the list of candidate bound constraints $\candidates$.
At each stage of the selection process, we incrementally post a suffix of the list of candidate bound constraints $\candidates$, i.e. each candidate bound constraint is only posted at most once during the search for the next candidate.
\begin{itemize}[label=--]
\item
Lines 1--5 split the list of candidate bound constraints $\candidates$ in a prefix and suffix part in an uneven way where the suffix is smaller than the prefix part.
The current way to split $\candidates$ was determined experimentally by testing different manners of partitioning on different examples.
We prefer to incrementally add a limited number of bound constraints so that we end up in a situation where we do not have enough constraints and this increases the number of backtracks needed to find a solution.
Otherwise, adding too many bound constraints would result in not increasing the number of backtracks, which would have the effect of scanning all remaining solutions in $\sols$ and deleting the added constraints in order to post a smaller set of constraints.
As a result, the same solution from $\sols$ would be generated several times, which can be mitigated by limiting the number of constraints we added.
\item
Line~6 performs a dichotomic search on the suffix and prefix parts of the candidate list to return a selected candidate $\newselected$ and the remaining list of candidate constraints $\newcandidates$.
\item
When we are at the top-level call of $\SelectOne$, i.e. when $\SelectOne$ is called in line~3 of Alg.~(\ref{alg::IncSelect}), lines 7--8 of Alg.~(\ref{alg::IncSelectOne}) remove any bound constraints posted within the call to Alg.~(\ref{alg::IncDicho}) on line~6 of Alg.~(\ref{alg::IncSelect}) to prepare for the next call to $\SelectOne$ from $\Select$.
\end{itemize}

\vspace{-0.5cm}
\begin{algorithm}
\caption{$\SelectOne(\toplevel,\sols,\allsols,\charvars,\vars,\candidates)$\label{alg::IncSelectOne}}
$\len\gets|\candidates|$\;
$\textbf{if}~\len>200~\textbf{then}~\midd\gets\len-100~
\textbf{else}\,\textbf{if}~\len<3~\textbf{then}~\midd\gets\lfloor\frac{\len+1}{2}\rfloor$\\
$\textbf{else}~\midd\gets\lfloor\frac{2\cdot\len+2}{3}\rfloor$\;
$\prefix\gets~\text{first}~\midd~\text{elements of}~\candidates$\;
$\suffix\gets~\text{last}~\len-\midd~\text{elements of}~\candidates$\;
$(\newselected,\hspace*{-1pt}\newcandidates)\hspace*{-1pt}\gets\hspace*{-1pt}\Dicho(\sols,\hspace*{-1pt}\allsols,\hspace*{-1pt}\charvars,\hspace*{-1pt}\vars,\hspace*{-1pt}\len,\hspace*{-1pt}\prefix,\hspace*{-1pt}\suffix)$\;
\If{$\toplevel$}{
    remove all bound constraints posted from the current call to $\Dicho$\;
}
\Return $(\newselected,\newcandidates)$\;
\end{algorithm}

\vspace{-0.5cm}
Alg.~(\ref{alg::IncDicho}) performs a dichotomic search wrt the suffix and prefix parts of the candidate list.
\begin{itemize}[label=--]
\item Line~1 posts all candidate bound constraints from $\suffix$.
\item For each solution $\sols$, line~2 computes the number of backtracks to obtain the next solution, and stops when the number of backtracks increases i.e.~$\missing=\textsc{true}$, or when the list of solutions is fully explored i.e.~$\missing=\textsc{false}$.
The set $\sols'$ corresponds to the set $\sols'$ from which we removed all the solutions leading to the same number of backtracks.
\item 
\begin{enumerate}
\item
If the number of backtracks increases and we can still add several bound constraints to the constraints to keep (see line~3), we continue the dichotomic search by using the bound constraints of the $\prefix$ to select the next bound to keep (see lines 4--5).
\item
If we can only add one bound constraint (see line~7), we return that bound constraint if the number of backtracks increases (see \textbf{then} part in line~8); otherwise we do not select a bound constraint (see \textbf{else} part in line~8).
\item
Otherwise, since adding all bound constraints of $\suffix$ (see line~1) was sufficient to avoid increasing the number of backtracks when searching for all solutions, we look for the next candidate bound constraint to select in the suffix (see line~10).
\end{enumerate}
\end{itemize}

\vspace{-0.5cm}
\begin{algorithm}
\caption{$\Dicho(\sols,\allsols,\charvars,\vars,\len,\prefix,\suffix)$\label{alg::IncDicho}}
$\forall\,\bound\in\suffix:~\text{post bound constraint}~\bound~\text{on}~\charvars$\;
$(\sols',\missing)\gets\Enumerate(\sols,\allsols,\charvars,\vars)$\;
\uIf{$\missing\land\len>1$}{
    $(\newselected,\hspace*{-2pt}\candidates)\hspace*{-2pt}\gets\hspace*{-2pt}\SelectOne(\textsc{false},\hspace*{-1pt}\sols',\allsols,\hspace*{-1pt}\charvars,\hspace*{-1pt}\vars,\hspace*{-1pt}\prefix)$\;
    \Return $(\newselected,\candidates\cup\suffix)$\;
    }
    remove all bound constraints posted on line~1\;
    \uIf{$\len=1$}{
        $\textbf{if}~\missing~\textbf{then}~\Return (\prefix,\emptyset)~\textbf{else}~\Return (\emptyset,\prefix)$
    } \Else{
        \Return $\SelectOne(\textsc{false},\sols,\allsols,\charvars,\vars,\suffix)$
    }
\end{algorithm}

\vspace{-0.5cm}
Alg.~\ref{alg::IncEnumerate} iteratively computes the next solution of each solution in $\sols$, considering the current set of posted bound constraints, until backtracking increases wrt the number of backtracks obtained using all bound constraints, or until no more solutions exist.
The \emph{next solution} of the $\isol$-th solution $\sol$ of the set $\sols$ is the smallest lexicographic solution strictly greater than $\sol$ (see lines~1--4).
If the current solution $\sol$ does not increase the number of backtracks (line~7), it is removed from the list of solutions $\sols$ to be checked (line~8), and the algorithm continues to check the remaining solutions $\sols'$ (line~10); otherwise, if backtracking increases, the check is terminated (line~12).

\vspace{-0.5cm}
\begin{algorithm}
\caption{$\Enumerate(\sols,\allsols,\charvars,\vars)$\label{alg::IncEnumerate}}

let $(\isol,-,\sol)$ be the first element of $\sols;~\jsol\gets \isol+1$\;
let $(\jsol,\nback,-)$ be the $\jsol$-th element of $\allsols;~\finished\gets\textsc{false}$\;
\If{$\isol>0\land~\text{post constraint}~\charvars >_\mathit{\ell\!ex} \sol~\text{fails}$}{
        $\nback\gets0;$ $\finished\gets\textsc{true}$\;
}
$\textbf{if}~\neg\finished~\textbf{then}~(\back,\finished,-)\gets\Labeling(\charvars,\vars)$\;
$\textbf{if}~\isol>0~\textbf{then}~\text{remove lexicographic constraint that was posted}$\;
\uIf{$\back=\nback$}{
    $\sols'\gets\sols-\{(I,Bi,Si)\}$\;
    \uIf{$\sols'\neq\emptyset$}{
        \Return $\Enumerate(\sols',\allsols,\charvars,\vars)$\;
    } \lElse {
        \Return $(\sols',\textsc{false})$
    }
} \lElse {
    \Return $(\sols,\textsc{true})$
}
\end{algorithm}

\section{Proofs for the Conjectures of the \partition\ Constraint}\label{sec:partition}

We borrow the definition of the \partition\ constraint from ~\cite{CheukamNgouonouGQBD25}.

\begin{definition}
$\partition([X_1,X_2,\ldots,X_\pn], \nval, \pmin, \pmax, \prange, \psumsquares)$ is satisfied iff
\begin{align}
  \nval        & = |\{X_1,X_2,\ldots,X_n\}|\hspace*{0.8cm}
  \psumsquares  \hspace*{3pt} = \sum_{\avalue \in \sequence} |\{i \mid X_i = \avalue\}|^2 \\
  \pmin        & = \min_{\avalue \in \sequence} |\{i \mid X_i = \avalue\}|\hspace*{1cm}
  \pmax          = \max_{\avalue \in \sequence} |\{i \mid X_i = \avalue\}|\hspace*{1cm}
  \prange      & = \pmax - \pmin
\end{align}
\end{definition}

We prove the following three selected conjectures, which were found by the Bound Seeker.

\begin{itemize}
\item
An upper bound on $\psumsquares$ and two distinct upper bounds on $\prange$:
  \begin{align}
  \psumsquares & \leq \mmid^2 + \sm \cdot \rr + \smin\label{eqn::upper_bound_sum_squares}
  \end{align} with: 
  
  \begin{tikzpicture}
\begin{scope}
\draw node[right,text width=7cm] {
\begin{equation}
\mmid  = \left\{\begin{array}{ll}
    \pmin + (\prmin\mbox{ \bf mod }\prange) & \mbox{if }  \prange > 0\\
    \pmin  & \mbox{otherwise}
    \end{array}\right. \label{mid}
\end{equation}
};
\end{scope}
\begin{scope}[xshift=6.5cm]
\draw node[right,text width=5cm] {
\begin{equation}
\prmin  = \pn - \nval \cdot \pmin \label{eqn::first_secondary_feature_upper_bound}
\end{equation}
};
\end{scope}
\end{tikzpicture}

\begin{tikzpicture}
\begin{scope}
\draw node[right,text width=5.5cm] {
\begin{equation}
 \rr  =  \left\{\begin{array}{ll}
    \left\lfloor\frac{\prmin}{\prange}\right\rfloor & \mbox{if }  \prange > 0\\
    0 & \mbox{otherwise}
    \end{array}\right.
\end{equation}
};
\end{scope}
\begin{scope}[xshift=6cm,yshift=0.4cm]
\draw node[right,text width=5cm] {
\begin{equation}
\sm  = \pmax^2-\pmin^2
\end{equation}
};
\end{scope}
\begin{scope}[xshift=6cm,yshift=-0.4cm]
\draw node[right,text width=5cm] {
\begin{equation}
\smin  = \pmin^2\cdot(\nval - 1)\label{eqn::last_secondary_feature_upper_bound}
\end{equation}
};
\end{scope}
\end{tikzpicture} 
\item
  Two distinct upper bounds on $\prange$:
  \begin{align}
  \prange & \leq \pn - \nval\cdot \pmin\label{eqn::upper_bound_range1}\\
  \prange & \leq \min(\nval\cdot\pmax-\pn,\pmax-1)\label{eqn::upper_bound_range2}
  \end{align}
\end{itemize}

We provide the proofs of correctness for three selected bounds~\eqref{eqn::upper_bound_sum_squares}, \eqref{eqn::upper_bound_range1}, and~\eqref{eqn::upper_bound_range2}.

To prove the conjecture \eqref{eqn::upper_bound_sum_squares}, we first prove the five following lemmas.

\begin{lemma}\label{lem:lemma1}

If there exist at least two partitions whose sizes are strictly between $\pmin$ and $\pmax$, in other words, if their sizes are $\pmin + r_1$ and $\pmin + r_2$ such that $1\leq r_1\leq r_2 < \prange$, then $\psumsquares$ is not maximal.
\end{lemma}

\begin{proof}We have
\begin{align}
(\pmin + r_1 - 1)^2 + (\pmin + r_2 + 1)^2 
& = (\pmin + r_1)^2 + \notag\\(\pmin + r_2)^2 + 2(r_2 - r_1) + 2
& >  (\pmin + r_1)^2 + \notag\\(\pmin + r_2)^2 \label{eqn::two_terms}
\end{align}
Let $O_i$ (with $i \in [1:\nval])$ be the sizes of the $\nval$ partitions of $\pn$ elements. It means that we have $\pn = \sum_i^{\nval}O_i$ and $\psumsquares = \sum_i^{\nval}O_i^2$. 
The two terms of~\eqref{eqn::two_terms} appear in the computation of $\psumsquares$.
In other words, it is possible to remove one element from partition $1$ and add it to partition $2$ to obtain a larger sum of squares without affecting the other terms in the computation of $\psumsquares$, since the value of $\sum_i^{\nval}O_i =\pn$ remains unchanged. This proves $S$ is not maximal.  \hfill $\square$
\end{proof}

\begin{lemma}\label{lem:lemma2}
Let $\omin$ and $\omax$ be the number of partitions that have respectively the size of $\pmin$ and $\pmax$.
If $(\prmin \bmod \prange > 0)$, then there is at least one partition  whose size is strictly between $\pmin$ and $\pmax$.
\end{lemma}

\begin{proof}By contradiction, suppose that $\prmin \bmod \prange > 0$ and all partitions are either of size $\pmin$ or $\pmax$.
By definition of $\prmin$ we have:
\begin{align}
    \prmin  = \sum_{i=1}^\nval (O_i - \pmin)
    = \sum_{i=1}^{\omax}(\pmax-\pmin) + \notag\\\sum_{i=1}^{\omin} (\pmin - \pmin)
    = \omax\cdot\prange
\end{align}
which contradicts that $\prmin \bmod \prange > 0$.
So $\prmin \bmod \prange > 0$ implies that there exists a partition whose size is strictly between $\pmin$ and $\pmax$. \hfill $\square$
\end{proof}

\begin{lemma}\label{lem:lemma3}
Let $\omin$ and $\omax$ be the number of partitions that have respectively the size of $\pmin$ and $\pmax$.
If $\psumsquares$ is maximal and $\prange > 0 \land (\prmin \bmod \prange > 0)$, then only one partition whose size is strictly between $\pmin$ and $\pmax$ exists, and its size is equal to $\pmin + \prmin \bmod \prange$.
\end{lemma}

\begin{proof}According to Lemma~\ref{lem:lemma1}, as $\psumsquares$ is maximal, there is at most one partition whose size is strictly between $\pmin$ and $\pmax$ and according to Lemma~\ref{lem:lemma2}, as $\prmin \bmod \prange > 0$ there is at least one partition whose size is strictly between $\pmin$ and $\pmax$. So there is only one partition  whose size is strictly between $\pmin$ and $\pmax$. So let $O_{\nval}$ be the size of that partition and $O_i, \forall i \in [1:\nval-1]$ the sizes of the remaining partitions. Let also $r^{*} = O_{\nval} -\pmin$.

Then we have $0 <r^{*} < \prange$ because $\pmin < O_{\nval} < \pmax$.
So, by definition, we have 
\begin{align}
    \prmin  = \sum_{i=1}^\nval (O_i - \pmin) = \sum_{i=1}^{\omax}(\pmax-\pmin) + \notag\\ \sum_{i=\omax+1}^{\omax+\omin} (\pmin - \pmin) +  r^{*}
     = \notag\\\omax\cdot\prange + r^{*} \mbox{ with } r^{*} < \prange \label{eqrmin}
\end{align}

According to the definition of Euclidean division, the relation \eqref{eqrmin} is equivalent to $r^{*} = \prmin \bmod \prange$. So $O_{\nval} = \pmin + \prmin \bmod \prange$.  \hfill $\square$
\end{proof}

\begin{lemma}\label{lem:lemma4}
Let $\omin$ and $\omax$ be the number of partitions that have respectively the size of $\pmin$ and $\pmax$.
If $\psumsquares$ is maximal and $\prange > 0 \land (\prmin \bmod \prange = 0)$, then there is no partition  whose size is strictly between $\pmin$ and $\pmax$.
\end{lemma}

\begin{proof}By contradiction, suppose that $\psumsquares$ is maximal, $\prange > 0 \land (\prmin \bmod \prange = 0)$ and there are $k$ $(k \geq 1)$ partitions  whose size is strictly between $\pmin$ and $\pmax$ and are denoted by $I_i, \forall i \in [1:k]$. Then, by definition
\begin{align}
   \prmin  = \sum_{i=1}^{\omax}(\pmax-\pmin) +  \sum_{i=1}^{k}(I_i-\pmin) = \omax\cdot\prange + \sum_{i=1}^{k}(I_i-\pmin) \label{eqrim1}\\\mbox{ with } I_i-\pmin < \prange, \forall i \in [1:k] \label{eqrim2}
\end{align}

Because $\prmin \bmod \prange = 0$, we have $\prmin = \floor{\dfrac{\prmin}{\prange}}\cdot \prange$. So according to \eqref{eqrim1}, $\forall i \in [1:k]$ with $I_i-\pmin < \prange$, we have 
\begin{equation}
  \omax\cdot\prange + \sum_{i=1}^{k}(I_i-\pmin) = \floor{\dfrac{\prmin}{\prange}}\cdot \prange \label{eqrim3} 
\end{equation}

So according to \eqref{eqrim3}, $\prange$ is a divisor of $\sum_{i=1}^{k}(I_i-\pmin)$. Which implies that $k \geq 2$, because of \eqref{eqrim2}. And according to Lemma~\ref{lem:lemma1}, $k \geq 2$ implies that $S$ is not maximal. Which is a contradiction.  \hfill $\square$
\end{proof}

\begin{lemma}\label{lem:lemma5}
Let $\omin$ and $\omax$ be the number of partitions that have respectively the size of $\pmin$ and $\pmax$.

Then  $\floor{\dfrac{\prmin}{\prange}}$ (resp.
    $\nval -\ceil{\dfrac{\prmin}{\prange}}$) is a tight upper bound of $\omax$ (resp. $\omin$). 
\end{lemma}

\begin{proof}We upper bound the values of $\omax$ by using the definition of $\pn$.
We have:
\begin{align}
    \pn  = \pmax\cdot\omax + \sum_{i=1}^{\nval-\omax}o_i \geq \pmax\cdot\omax + \notag\\(\nval-\omax) \cdot \pmin  \geq (\pmax - \pmin) \cdot\omax + \nval \cdot \pmin \\ 
   \Longrightarrow  \omax \leq \dfrac{n-\nval \cdot \pmin}{\pmax-\pmin}  \\
\intertext{And by definition, $\prmin = n-\nval \cdot \pmin$ and $\prange = \pmax-\pmin$. So:}
    \omax  \leq \dfrac{\prmin}{\prange}
\end{align}
Symmetrically, we upper bound $\omin$:
\begin{align}
    \pn  = \pmin\cdot\omin  + \sum_{i=1}^{\nval-\omin}o_i \leq \pmin\cdot\omin + \notag\\\pmax\cdot(\nval-\omin) = - (\pmax - \pmin) \cdot \omin + \pmax\cdot\nval\\
  \Longrightarrow  \omin \leq \dfrac{\pmax\cdot\nval - \pn}{\pmax-\pmin} = \dfrac{(\prange+\pmin)\cdot\nval-\pn}{\prange} = \notag\\ \nval + \dfrac{-(\pn - \pmin\cdot\nval)}{\prange} =  \nval -\dfrac{\prmin}{\prange} 
\end{align}
Since $\omax$, $\omin \in \mathbb{N}$, we have
\begin{equation}
    \omax \leq \floor{\dfrac{\prmin}{\prange}}\qquad \mbox{and} \qquad
    \omin \leq \nval -\ceil{\dfrac{\prmin}{\prange}}\label{I1}
\end{equation}
Finally, according to Lemmas~\ref{lem:lemma3} and~\ref{lem:lemma4}, when $\psumsquares$ is maximal, we have $\prmin=\omax\cdot\prange+\prmin \bmod \prange$ and at most one partition exists whose size is strictly between $\pmin$ and $\pmax$.
Which means that if $\psumsquares$ is maximal, we have:
\begin{align}
\omax & =  \floor{\dfrac{\prmin}{\prange}}\\\Longrightarrow
\omin & =  \nval - \omax - \alpha =\nval - \floor{\dfrac{\prmin}{\prange}} - \alpha \\
\mbox{ with } \alpha = 0 \mbox{ if } & \prmin \bmod \prange = 0 \mbox{ and } \alpha = 1 \mbox{ if not}.\\\Longrightarrow
\omin & =  \nval - \floor{\dfrac{\prmin}{\prange}} - \alpha = \nval -\ceil{\dfrac{\prmin}{\prange}}
\end{align} \hfill $\square$
\end{proof}

\subsection{Conjecture~\eqref{eqn::upper_bound_sum_squares}}
\begin{proof}[Conjecture \eqref{eqn::upper_bound_sum_squares}]
\begin{itemize}
\item When $\prange = \pmax - \pmin = 0$. The sizes of the partitions are all the same. 
Then, according to equations \eqref{eqn::first_secondary_feature_upper_bound} to \eqref{eqn::last_secondary_feature_upper_bound} we have $\mmid = \pmin, \sm = \rr = 0, \smin = \pmin^2\cdot(\nval-1)$. By substituting these $\mmid, \sm , \rr$ and $\smin$ in \eqref{eqn::upper_bound_sum_squares}, we obtain 
\begin{align}
  \psumsquares & \leq \pmin^2 + 0 + \pmin^2\cdot(\nval-1) = \pmin^2\cdot\nval
\end{align}
which is consistent with the definition of $\psumsquares$ because, as the sizes of the $\nval$ partitions are all the same, they are equal to $\pmin$. Which means that $\psumsquares = \pmin^2\cdot\nval$.
\item When $\prange > 0 \land (\prmin \bmod \prange > 0)$.
Let $\omin$ and $\omax$ be the number of partitions that have respectively the size of $\pmin$ and $\pmax$.

According to Lemmas~\ref{lem:lemma3} and~\ref{lem:lemma5}, the maximal value $\psumsquares^{*}$  of $\psumsquares$ is  

\begin{align}
\psumsquares^{*} = \pmax^2\cdot\omax + \pmin^2\cdot\omin &+ (\pmin + \prmin \bmod \prange)^2 \label{maxS}\\\mbox{ with }
\omax = \floor{\dfrac{\prmin}{\prange}} \mbox{ and }&\omin = \nval -\floor{\dfrac{\prmin}{\prange}}-1\notag
\end{align}

So because  $\prmin \bmod \prange > 0$, we have according to equations \eqref{eqn::first_secondary_feature_upper_bound} to \eqref{eqn::last_secondary_feature_upper_bound}, $\mmid = \pmin + \prmin \bmod \prange, \sm = \pmax^2 -\pmin^2, \rr = \floor{\dfrac{\prmin}{\prange}}, \smin = \pmin^2\cdot(\nval-1)$. By substituting these $\mmid, \sm , \rr$ and $\smin$ in \eqref{eqn::upper_bound_sum_squares}, we obtain $\psumsquares\leq \psumsquares^{*}$ according to \eqref{maxS}, which is consistent.

\item When $\prange > 0 \land (\prmin \bmod \prange = 0)$. Let $\omin$ and $\omax$ be the number of partitions that have respectively the size of $\pmin$ and $\pmax$.

According to Lemmas~\ref{lem:lemma4} and~\ref{lem:lemma5}, the maximal value $\psumsquares^{*}$  of $\psumsquares$ is  

\begin{align}
\psumsquares^{*} = \pmax^2\cdot\omax + \pmin^2\cdot\omin & \label{maxS1}\\\mbox{ with }
\omax = \floor{\dfrac{\prmin}{\prange}} \mbox{ and }&\omin = \nval -\floor{\dfrac{\prmin}{\prange}}\notag
\end{align}

So because  $\prmin \bmod \prange = 0$, we have  according to equations \eqref{eqn::first_secondary_feature_upper_bound} to \eqref{eqn::last_secondary_feature_upper_bound}, $\mmid = \pmin, \sm = \pmax^2 -\pmin^2, \rr = \floor{\dfrac{\prmin}{\prange}}, \smin = \pmin^2\cdot(\nval-1)$. By substituting these $\mmid, \sm , \rr$ and $\smin$  in \eqref{eqn::upper_bound_sum_squares}, we obtain $\psumsquares\leq \psumsquares^{*}$ according to \eqref{maxS1}, which is consistent. \hfill $\square$
\end{itemize}
\end{proof}

\subsection{Conjecture~\eqref{eqn::upper_bound_range1}}
\begin{proof}[Conjecture \eqref{eqn::upper_bound_range1}]
\begin{itemize}

\item If $\nval = 1$, then $\pmin = \pmax = \pn$. Thus $\prange = 0$ and  $\pn - \nval \cdot \pmin = \pn - \pmin = 0$. Which implies that $\prange \leq \pn - \nval \cdot \pmin$.

    \item If $\nval \geq 2$, then the number of values is equal to the size of the largest partition plus the size of the smallest partition plus the size of all the other partitions.

\begin{align}
\pn & = \pmax + \pmin + \sum_{i=1}^{\nval-2} O_i
\intertext{As $\pmin$ is the size of the smallest partition, we have $O_i \geq \pmin$.}
\pn & \geq \pmax + \pmin + \sum_{i=1}^{\nval-2} \pmin = \pmax + (\nval - 1) \cdot \pmin\\
\pn - \nval \cdot \pmin & \geq \pmax - \pmin
\end{align}
Using the definition $\prange = \pmax - \pmin$ we obtain $\pn - \nval \cdot \pmin \geq \prange$.

\item Tightness of the conjecture~\eqref{eqn::upper_bound_range1}: We can construct for every possible value of $\pn,\nval$ and $\pmin$ the set of partitions so that $\pn - \nval \cdot \pmin = \prange$, by setting only one of the partitions to size $\pmax = \pn - (\nval-1)\cdot\pmin$ and the rest to size $\pmin$. Because, in that case, we have:
\begin{align}
\pn & = \pmax + (\nval-1)\cdot\pmin \\
 \pn - \nval \cdot \pmin & = \pmax + (\nval-1)\cdot\pmin - \nval \cdot \pmin = \pmax - \pmin =\prange
\end{align}
\end{itemize} \hfill $\square$
\end{proof}

\subsection{Conjecture~\eqref{eqn::upper_bound_range2}}
\begin{proof}[Conjecture \eqref{eqn::upper_bound_range2}]\begin{itemize}
\item If $\nval = 1$, then $\pmin = \pmax = \pn$. Thus $\prange = 0$ and $\nval \cdot \pmax - \pn = \pmax - \pn= 0$. Which implies that $\prange \leq \nval \cdot \pmax - \pn$. And because $\pmax \geq 1$, we also have $\prange = 0$ and $0\leq \pmax-1$. Which implies that $\prange \leq \pmax-1$.
Since two quantities bound $\prange$, the smallest of them bounds $\prange$. Hence $\prange \leq \min(\nval \cdot \pmax - \pn, \pmax - 1)$.
\item If $\nval \geq 2$, then we first show $\prange \leq \nval \cdot \pmax - \pn$.
\begin{align}
\pn & = \pmax + \pmin + \sum_{i=1}^{\nval-2} O_i\\
\pn & \leq \pmax + \pmin + \sum_{i=1}^{\nval-2} \pmax = (\nval - 1) \pmax + \pmin\\
\pmax - \pmin & \leq \nval \cdot \pmax - \pn\\
\prange & \leq \nval \cdot \pmax - \pn
\end{align}
The largest range one can obtain is when one element is alone in a partition and the remaining $\pn -1$ elements are together in the 2nd partition. We have $\prange \leq \pmax - 1$.
Since $\prange$ is bounded by two quantities, it is bounded by the smallest one, hence $\prange \leq \min(\nval \cdot \pmax - \pn, \pmax - 1)$.

\item Tightness of the conjecture~\eqref{eqn::upper_bound_range2}:

For the case $\nval = 1$, we have $\pmin = \pmax = \pn$. Thus $\prange = 0$ and $\nval \cdot \pmax - \pn = \pmax - \pn= 0$. Which implies that $\prange = \nval \cdot \pmax - \pn = 0$. So we have $0 = \min(\nval \cdot \pmax - \pn, \pmax - 1)$. Which implies that $\prange = \min(\nval \cdot \pmax - \pn, \pmax - 1)$. So, the bound is tight. 

For the case of $\nval \geq 2$, we can construct for every possible value of $\pn,\nval$ and $\pmax$ the set of partitions with $\nval \cdot \pmax - \pn = \prange$ or $\pmax-1=\prange$, either by setting only one of the partitions to size $\pmin = \pn - (\nval-1)\cdot\pmax$ and the rest to size $\pmax$ if $\pn > (\nval-1)\cdot\pmax$ or either  by setting one of the partitions to size $1$ if $\pn \leq (\nval-1)\cdot\pmax$. Because  we have:

\begin{itemize}
 \item If $\pn > (\nval-1)\cdot\pmax$, then $\pn  = \pmin + (\nval-1)\cdot\pmax$.
Which implies that 
$\mbox{\hspace{0.2cm}} \nval \cdot \pmax - \pn = \nval \cdot \pmax - \pmin - (\nval-1)\cdot\pmax$.  
So \hspace{0.2cm}$\nval \cdot \pmax - \pn = \pmax - \pmin =\prange$.

 \item If $\pn \leq (\nval-1)\cdot\pmax$, we have $\nval > 2$ because if $P=2$, it implies that $\pn \leq \pmax$. This implies that $\pn = \pmax$, meaning that $P = 1$, which is inconsistent with $P=2$. So to reach the tightness, we set one partition to size $\pmax$ and another to size $1$. And because $\nval > 2$ in this case, we can set the remaining $\nval -2$ partitions to size $\floor{\dfrac{\pn-\pmax-1}{\nval -2}}$ and size $\ceil{\dfrac{\pn-\pmax-1}{\nval -2}}$. Indeed, we have $1\leq\floor{\dfrac{\pn-\pmax-1}{\nval -2}}<\pmax$. Because first, $\pn-\pmax-1$ is the remaining number of elements to partition into $\nval-2$ non-empty sets. So $\nval-2\leq \pn-\pmax-1$, which leads to $1\leq\floor{\dfrac{\pn-\pmax-1}{\nval -2}}$. And second, we also have  $\pn \leq (\nval-1)\cdot\pmax$ equivalent to  \begin{align}
    \pn-\pmax-1 & \leq (\nval-1)\cdot\pmax- \pmax-1\\
    \dfrac{\pn-\pmax-1}{\nval -2} & \leq  \dfrac{(\nval-2)\cdot\pmax - 1}{\nval -2} \leq \pmax - \dfrac{1}{\nval -2}< \pmax\\ \Rightarrow\floor{\dfrac{\pn-\pmax-1}{\nval -2}} & < \pmax
 \end{align}
\end{itemize}
\end{itemize} \hfill $\square$
\end{proof}

\section{Proofs for the Conjectures of the \group\ Constraint}\label{sec:group}

We borrow the definition of the \group\ constraint from~\cite{CheukamNgouonouGQBD25}.

\begin{definition}
The $\group([X_1,X_2,\ldots,X_\pn],\gnv,\gng,\gmin,\gmax, \notag\grange,\gsumsquares,\dmin,\dmax,\drange,$ $\dsumsquares)$
constraint is satisfied iff
\begin{itemize}[label=\textbullet]
\item $X_1,X_2,\ldots,X_\pn$ is a sequence of 0/1,
\item $\gnv$ is the number of values $1$ in the sequence,
\item $\gng$ is the number of stretches of 1s,
\item $\gmin$ (resp. $\gmax$) is the length of the smallest (resp. longest) stretch of 1s,
\item $\grange$ is the difference between the lengths of the longest and the smallest stretch,
\item $\gsumsquares$ is the sum of the squared lengths of the stretches of 1s,
\item $\dmin$ (resp. $\dmax$) is the length of the smallest (resp. longest) inter-distance of 0s,
\item $\drange$ is the difference $\dmax-\dmin$,
\item $\dsumsquares$ is the sum of the squared lengths of the inter-distances of 0s.
\end{itemize}
When there is no stretch, $\gmin\hspace*{-1pt}=\hspace*{-1pt}\gmax\hspace*{-1pt}=\hspace*{-1pt}0$; when there is no inter-distance, $\dmin\hspace*{-1pt}=\hspace*{-1pt}\dmax\hspace*{-1pt}=\hspace*{-1pt}0$.
\end{definition}

\subsection{Conjecture~\eqref{eqn::upper_bound_nv}} 
We prove the selected conjecture~\eqref{eqn::upper_bound_nv}, which was found by the Bound Seeker.

\begin{align}
    \gnv \leq \min(\gng \cdot \gmax,\gn - \gng + 1) \label{eqn::upper_bound_nv}
\end{align}

\begin{proof}[Conjecture \eqref{eqn::upper_bound_nv}]
Let $g_i$ (with $i\in [1:\gng]$) be the number of $1$ in the $i$-th stretch of 1s. We have \begin{align}
    \gnv = \sum_{i=1}^\gng g_i\leq \gng\cdot\gmax \label{binseq:equation1}
\end{align}

If $\gng = 0$ then no stretch of 1s appears in the binary sequence.
Which means that $\gnv = 0 = \min(0,\gn - \gng + 1) = \min(\gng \cdot \gmax,\gn - \gng + 1)$.

If $\gng \geq 1$ then there are $\gng - 1$ inter-distances of 0s of lengths at least equal to $1$. This means that $\gnv \leq \gn - (\gng - 1)$. Also, thanks to \eqref{binseq:equation1} we have $\gnv \leq \gng\cdot\gmax$. So $\gnv \leq \min(\gng \cdot \gmax,\gn - \gng + 1)$. \hfill $\square$
\end{proof}

\subsection{Conjecture~\eqref{eqn::lower_bound_smax}}
We prove the selected conjecture \eqref{eqn::lower_bound_smax}, which was found by the Bound Seeker.

\begin{align}
    \gmax \geq \left\lfloor\dfrac{\gn}{\gn - \gnv + 1}\right\rfloor \label{eqn::lower_bound_smax}
\end{align}

\noindent
\begin{proof}[Conjecture \eqref{eqn::lower_bound_smax}]
Because $\gng$ is the number of stretches of 1s, then $\gng - 1$ is the number of inter-distances of 0s. So the minimum number of 0s in the binary sequence is $\gng - 1$. Therefore we have \begin{align}
    \gn \geq \gnv + \gng - 1 \Longleftrightarrow  \gn -\gnv + 1 \geq  \gng 
\end{align}

\jomodif{If $\gmax > 0$}, thanks to \eqref{binseq:equation1}, we have $\gng \geq \dfrac{\gnv}{\gmax}$ which leads to  \begin{align}
    \gn -\gnv + 1 \geq  \gng \geq  \dfrac{\gnv}{\gmax} \Longrightarrow \gmax \geq \dfrac{\gnv}{\gn -\gnv + 1}\end{align}\jomodif{And because $\gmax$ is an integer, we have $\gmax \geq \left\lceil\dfrac{\gnv}{\gn -\gnv + 1}\right\rceil$. According to \cite{10.5555/562056} for the positive integers $n$ and $m$ with $m\in \mathbb{N}^{*}$, we have \begin{align}
   \left\lceil\dfrac{n}{m}\right\rceil = \left\lfloor\dfrac{n + m - 1}{m}\right\rfloor\label{eq:floor-ceil}
\end{align} So according to \eqref{eq:floor-ceil}} \begin{align}
     \gmax \geq \left\lfloor\dfrac{\gnv + \gn -\gnv + 1 - 1}{\gn -\gnv + 1}\right\rfloor = \left\lfloor\dfrac{\gn}{\gn -\gnv + 1}\right\rfloor
\end{align}

\jomodif{If $\gmax = 0$, then $\gnv = 0$. So we have \begin{align}\gmax = 0 = \left\lfloor\dfrac{\gn}{\gn + 1}\right\rfloor = \left\lfloor\dfrac{\gn}{\gn -\gnv + 1}\right\rfloor\end{align}}\hfill $\square$
\end{proof}

\subsection{Conjecture~\eqref{eqn::upper_bound_smax}}
We prove the selected conjecture \eqref{eqn::upper_bound_smax}, which was found by the Bound Seeker.

\begin{align}
    \gmax \leq \left\{\begin{array}{ll}\gn + \grange & \mbox{ if } \grange = \gn\cdot\drange\\\\
    \left\lfloor\dfrac{\gn - \grange - \drange - \min(\drange,1)-1}{\min(\drange,1) + 2}\right\rfloor + \grange& \mbox{ otherwise} \end{array}\right. \label{eqn::upper_bound_smax}
\end{align}

\begin{proof}[Conjecture \eqref{eqn::upper_bound_smax}]
\jomodif{If $\grange = \gn\cdot\drange$, by definition $\gmax \leq \gn$. And as $\grange \geq 0$, we have  $\gmax \leq \gn + \grange$.}

Otherwise, if $\grange \neq \gn\cdot\drange$, we consider the case  when $\drange = 0 $ and \jomodif{$\grange \geq 1$} as well as the case when $\drange \geq 1$:
\begin{itemize}
    \item In the case $\drange = 0 $ and \jomodif{$\grange \geq 1$}, we have $\min(\drange,1) = 0$ and $\gmax > \gmin$. This means that there are at least two stretches of $1$ and an \jomodif{inter-distance} of 0s between them in the binary sequence. So the binary sequence has at least \jomodif{two stretches of lengths $\gmin$ and $\gmax$ because $\grange \geq 1$. This also means that it has an inter-distance of length at least equal to} $1$, which means that $\gmin \leq \gn - \gmax - 1$. Then we have \begin{align}
        \gmin \leq \gn - \gmax - 1 & \iff 2\cdot\gmin \leq \gn - \gmax + \gmin - 1\\
         & \iff 2\cdot\gmin \leq \gn - \grange - 1\\
        & \implies \gmin \leq \left\lfloor\dfrac{\gn - \grange-1}{2}\right\rfloor\\
        & \implies \gmax \leq \left\lfloor\dfrac{\gn - \grange-1}{2}\right\rfloor + \grange \\
        & \iff \gmax \leq \left\lfloor\dfrac{\gn - \grange - \drange - \min(\drange,1)-1}{\min(\drange,1) + 2}\right\rfloor + \grange
    \end{align}

    \item In the case  $\drange \geq 1$, the binary sequence has at least two \jomodif{inter-distances} of 0s. This means that there are also at least three stretches of 1s. \jomodif{So we have \begin{align}\gn - \dmax - \dmin - 2\cdot\gmin-\gmax \geq 0\\\gn - (\drange + \dmin) - \dmin - 2\cdot\gmin-\gmax \geq 0\\\gn - \drange  -2\cdot\dmin - 2\cdot\gmin-\gmax \geq 0\\
    2\cdot\gmin \leq \gn - \drange - \gmax - 2\cdot\dmin\end{align} As $\dmin\geq1$, we have $-2\cdot\dmin\leq -2$. So we have \begin{align}
2\cdot\gmin \leq \gn - \drange - \gmax - 2\cdot\dmin \leq  \gn - \drange - \gmax - 2
    \end{align}}This leads to $2\cdot\gmin \leq \gn - \drange - \gmax - 2$\jomodif{, which leads to}
   \begin{align}
         3\cdot\gmin \leq \gn - \drange - \gmax + \gmin - 2
        \Longleftrightarrow 3\cdot\gmin \leq \gn - \drange - \grange - 2\\
        \Longrightarrow \gmin \leq \left\lfloor\dfrac{\gn - \drange - \grange - 2}{3}\right\rfloor\\
        \Longrightarrow \gmin \leq \left\lfloor\dfrac{\gn - \drange - \grange - 2}{3}\right\rfloor = \left\lfloor\dfrac{\gn - \grange - \drange - \min(\drange,1)-1}{\min(\drange,1) + 2}\right\rfloor\\ \Longrightarrow \gmax \leq  \left\lfloor\dfrac{\gn - \grange - \drange - \min(\drange,1)-1}{\min(\drange,1) + 2}\right\rfloor + \grange
    \end{align}
\end{itemize} \hfill $\square$
\end{proof}

\subsection{Conjecture~\eqref{eqn::upper_bound_dmin}}
We prove the selected conjecture \eqref{eqn::upper_bound_dmin}, which was found by the Bound Seeker.

\begin{align}
    \dmin \leq \left\{\begin{array}{ll}0 & \mbox{ if } \gng \leq 1\\\\
    \left\lfloor\dfrac{\gn - \gmax + 1 - \gng}{\gng - 1}\right\rfloor& \mbox{ if } \gng > 1\end{array}\right. \label{eqn::upper_bound_dmin}
\end{align}

\begin{proof}[Conjecture \eqref{eqn::upper_bound_dmin}]
If $\gng \leq 1$ then there is no \jomodif{inter-distance of 0s} between stretches of 1s. This means that $\dmin = 0$.

If $\gng > 1$ then there is a stretch of 1s that has a size equal to $\gmax$ and there are $\gng - 1$ stretches of 1s that have a size at least equal to $1$. Let $N_0$ be the number of 0s, which are between stretches of 1s. So we have $N_0 \leq \gn - \gmax - (\gng - 1)$. Let~$d_i$ (with $i \in [1:\gng - 1]$) be the number of $0$s of the $i$-th \jomodif{inter-distance} between stretches of 1s. We have \begin{align}
    N_0 = \sum_{i=1}^{\gng-1} d_i\geq (\gng -1)\cdot\dmin & \implies \dmin \leq \dfrac{N_0}{\gng-1} \leq \dfrac{\gn - \gmax - (\gng - 1)}{\gng-1}\\
    & \implies \dmin \leq \left\lfloor\dfrac{\gn - \gmax - (\gng - 1)}{\gng-1}\right\rfloor
\end{align} \hfill $\square$
\end{proof}

\subsection{Conjecture~\eqref{eqn::upper_bound_dmax}}
We prove the selected conjecture \eqref{eqn::upper_bound_dmax}, which was found by the Bound Seeker.

\begin{align}
   \dmax \leq [\gng \geq 2] \cdot (\gn - \gng \cdot \gmin - \gng + 2) \label{eqn::upper_bound_dmax}
\end{align}

\begin{proof}[Conjecture \eqref{eqn::upper_bound_dmax}]
If $\gng \leq 1$ then there is no \jomodif{inter-distance of 0s} between stretches of 1s. This means that $\dmax = 0$.

If $\gng \geq 2$ then there are $\gng$ stretches of 1s of lengths at least equal to $\gmin$. There is also an \jomodif{inter-distance of 0s} of length $\dmax$ and there are $\gng - 2$ \jomodif{inter-distances of 0s} of lengths at least equal to one. This means that $\dmax \leq \gn - \gng \cdot \gmin - (\gng - 2) = \gn - \gng \cdot \gmin - \gng + 2$.
\hfill $\square$
\end{proof}

\subsection{Conjecture~\eqref{eqn::lower_bound_ssquare}}
We prove the selected conjecture \eqref{eqn::lower_bound_ssquare}, which was found by the Bound Seeker.

\begin{align}
   \gsumsquares \geq \gmin^2\cdot\gng\label{eqn::lower_bound_ssquare}
\end{align}

\begin{proof}[Conjecture \eqref{eqn::lower_bound_ssquare}]
 \begin{align}
   \gsumsquares = \sum_{i=1}^\gng g_i^2 \geq \gmin^2\cdot\gng
\end{align} \hfill $\square$
\end{proof}

\subsection{Conjecture~\eqref{eqn::lower_bound_ssquare2}}
We prove the selected conjecture \eqref{eqn::lower_bound_ssquare2}, which was found by the Bound Seeker.

\begin{align}
   \gsumsquares \geq \grange\cdot(\grange + 1)\cdot\min(\gng,1) + \grange + \gng\label{eqn::lower_bound_ssquare2}
\end{align}

\begin{proof}[Conjecture \eqref{eqn::lower_bound_ssquare2}]
If $\gng = 0$ then $\grange = 0$, which leads to $\gsumsquares = 0$; \jomodif{hence, the property holds}.

If $\gng \geq 1$, then $\gmin \geq 1$.
 \begin{align}
  \gmin \geq 1 \Longleftrightarrow \gmin - 1 \geq 0\\
   \Longleftrightarrow \gmax\cdot(\gmin - 1) \geq 0\\
    \Longleftrightarrow \gmax\cdot(\gmin - 1) + \gmin  \geq 1\\
    \Longleftrightarrow -1 \geq -\gmax\cdot(\gmin - 1) - \gmin\\
    \mbox{\jomodif{ Multipliying the inequality  by 2, we have }}\notag\\ -2 \geq -2\cdot\gmax\cdot\gmin + 2\cdot(\gmax - \gmin)\\\mbox{\jomodif{Adding $\gmax^2 + \gmin^2 + \gng$ on each part of the inequality, we obtain }}\notag\\
     \gmax^2 + \gmin^2 + \gng - 2 \geq \gmax^2 + \gmin^2 - 2\cdot\gmax\cdot\gmin + 2\cdot(\gmax - \gmin) + \gng\\\mbox{\jomodif{Replacing $\gmax^2 + \gmin^2 - 2\cdot\gmax\cdot\gmin$ by $(\gmax - \gmin)^2$, we obtain }}\notag\\
     \gmax^2 + \gmin^2 + \gng - 2 \geq (\gmax - \gmin)^2  + 2\cdot\grange + \gng\\\mbox{\jomodif{Replacing $\gmax - \gmin$ by $\grange$, we obtain }}\notag\\
     \gmax^2 + \gmin^2 + \gng - 2 \geq \grange^2  + 2\cdot\grange + \gng  =  \grange\cdot(\grange + 1) + \grange + \gng\\\mbox{\jomodif{As  $\gng \geq 1$, we have $\min(\gng,1)=1$. This leads to }}\notag\\
      \gmax^2 + \gmin^2 + \gng - 2 \geq \grange\cdot(\grange + 1)\cdot\min(\gng,1) + \grange + \gng \label{eq:binseq2}
\end{align}\jomodif{Let $g_i$ (with $i \in [1:\gng]$) be the number of $1$ in the $i$-th stretch of 1s.
Note that with $g_{\gng-1} = \gmax$ and $g_{\gng} = \gmin$, we have $\forall i \in [1:\gng - 2], g_i \geq 1$, which leads to}\begin{align}\jomodif{\sum_{i=1}^{\gng-2} g_i^2 \geq \sum_{i=1}^{\gng-2}1 =  \gng - 2}\\\mbox{\jomodif{Adding $\gmax^2 + \gmin^2$ on each part of the inequality, we obtain }}\notag\\
\jomodif{\sum_{i=1}^\gng g_i^2 = \gmax^2 + \gmin^2 + \sum_{i=1}^{\gng-2} g_i^2 \geq \gmax^2 + \gmin^2 + \gng - 2}\\
    \Longleftrightarrow \gsumsquares = \sum_{i=1}^\gng g_i^2 = \gmax^2 + \gmin^2 + \sum_{i=1}^{\gng-2} g_i^2 \geq \gmax^2 + \gmin^2 + \gng - 2 \label{eq:binseq3}
\end{align}Finally, thanks to \eqref{eq:binseq2}  and \eqref{eq:binseq3} we have \begin{align}
   \gsumsquares \geq  \gmax^2 + \gmin^2 + \gng - 2 \geq \grange\cdot(\grange + 1)\cdot\min(\gng,1) + \grange + \gng
\end{align} \hfill $\square$
\end{proof}

\subsection{Conjecture~\eqref{eqn::lower_bound_ssquare3}}
We prove the selected conjecture \eqref{eqn::lower_bound_ssquare3}, which was found by the Bound Seeker.

\begin{align}
   \gsumsquares \geq \max(\gmax^2 + 1-[\dmin=0]-[\gmax=0],0)\label{eqn::lower_bound_ssquare3}
\end{align}

\begin{proof}[Conjecture \eqref{eqn::lower_bound_ssquare3}]
If $\dmin=0\wedge\gmax=0$, then $\gsumsquares = 0 = \max(-1,0) = \max(\gmax^2 + 1-[\dmin=0]-[\gmax=0],0)$.

If $\dmin=0\wedge\gmax\geq 1$, then there is just one stretch of 1s in the binary sequence. In that case $\gsumsquares = \gmax^2 = \gmax^2 + 1-[\dmin=0]-[\gmax=0] = \max(\gmax^2 + 1-[\dmin=0]-[\gmax=0],0)$.

If $\dmin\geq 1\wedge\gmax\geq 1$, there are at least two stretches of 1s: one of length $\gmax$ and another of length at least $1$.
Therefore, $\gsumsquares \geq \gmax^2 + 1 = \gmax^2 + 1-[\dmin=0]-[\gmax=0] = \max(\gmax^2 + 1-[\dmin=0]-[\gmax=0],0)$.
\hfill $\square$
\end{proof}

\subsection{Conjecture~\eqref{eqn::upper_bound_ssquare}}
We prove the selected conjecture \eqref{eqn::upper_bound_ssquare}, which was found by the Bound Seeker.

\begin{align}
\gsumsquares \leq \left\{\begin{array}{ll}\max(\gnv^2 + \gng - 1,0)& \mbox{if }\gng \leq 1\\\max((\gnv - \gng + 1)^2 + \gng - 1,0)& \mbox{otherwise}\end{array}\right.\label{eqn::upper_bound_ssquare}
\end{align}
To prove this conjecture, we first prove the following Theorem~\ref{the:maximisation}.

\begin{theorem}
    [maximisation of $\psumsquares = \sum_i^{\nval}y_i^2$]\label{the:maximisation}
Let $y_1,y_2,\cdots,y_{\nval}$ be non-negative integers whose sum is equal to $\pn$ and which maximise $\psumsquares=\sum_{i=1}^{\nval}y_i^2$. Then the \jomodif{largest} integer is equal to $y_1 = \pn-(P-1)$ and the $P-1$ \jomodif{remaining} integers are all equal to $1$. 
\end{theorem}

\begin{proof}
Let be the distribution of values among the integers $y_i$ \jomodif{with } $\forall i \in [1:P]$ such that the maximum integer has the value $\pn-(P-1)$ and the other integers are all equal to $1$. So the sum of squares of the integers of this distribution is equal to $S_0 = (n-(P-1))^2 + P - 1$. We will now show that for any other distribution of values among the integers \jomodif{ $y_i$, we have $S \leq S_0$:}

Any distribution other than the one that gives $S_0$ can be obtained by removing $m$ occurrences of $1$ from the largest value $y_1$ of the distribution for $S_0$, and then distributing these $m$ values to the other initially equal $1$ values.
So we have $\forall i \in [2:P], j_i \in \mathbb{N}$, \begin{align}
   S = (y_1 - m)^2 + \sum_{i=2}^P (1 + j_i)^2 \mbox{ with } \sum_{i=2}^P j_i = m < y_1\\
    S = y_1^2 + m^2 - 2\cdot m\cdot y_1+ \sum_{i=2}^P(1 + j_i^2 + 2\cdot j_i)\\
    \jomodif{ S = y_1^2 + m^2 - 2\cdot m\cdot y_1+ \sum_{i=2}^P 1 + \sum_{i=2}^Pj_i^2 + 2\cdot\sum_{i=2}^P j_i}\\\jomodif{\mbox{As  }\sum_{i=2}^P 1 = P-1\mbox{ and }\sum_{i=2}^P j_i = m, \mbox{ we have :} }\notag\\
    S = y_1^2 + m^2 - 2\cdot m\cdot y_1+ P - 1 + 2\cdot m + \sum_{i=2}^P j_i^2\\\jomodif{\mbox{After rearranging each term, we obtain :}}\notag\\
     S = y_1^2 + P - 1 +  m^2 - 2\cdot m\cdot y_1 + 2\cdot m + \sum_{i=2}^P j_i^2 \\\jomodif{\mbox{As }S_0 =  y_1^2 + P - 1,\mbox{ we have :}}\notag\\
     S =  S_0 + \left(m^2 - 2\cdot m\cdot y_1 + 2\cdot m + \sum_{i=2}^P j_i^2\right)\label{eq:binseq4}
\end{align}
We now show that \jomodif{the term in the parenthesis of \eqref{eq:binseq4} is negative. That means } \begin{align}m^2 - 2\cdot m\cdot y_1 + 2\cdot m + \jomodif{\sum_{i=2}^P j_i^2} \leq 0\label{theo:part1}\end{align}

For that, \jomodif{we first express $2\cdot m\cdot y_1$ with $m^2$} : 
\begin{align}
    y_1 = y_1 - m + m \Longleftrightarrow 2\cdot m\cdot y_1 = 2\cdot m\cdot (y_1 - m + m)\notag\\
    \Longleftrightarrow 2\cdot m\cdot y_1 = 2\cdot m\cdot(y_1 -m) + 2\cdot m^2\label{theo:part2}\end{align}
  \jomodif{Then, according to \eqref{theo:part2}, we replace $2\cdot m\cdot y_1$ by $2\cdot m\cdot(y_1 -m) + 2\cdot m^2$ in the left term $m^2 - 2\cdot m\cdot y_1 + 2\cdot m + \sum_{i=2}^P j_i^2$ of the inequality \eqref{theo:part1}}. So we have \begin{align}  
    m^2 - 2\cdot m\cdot y_1 + 2\cdot m + \sum_{i=2}^P j_i^2  = 2\cdot m -2\cdot m\cdot(y_1 -m)- m^2 + \sum_{i=2}^P j_i^2\label{theo:part3}\end{align}\jomodif{Now we show that the two parts $2\cdot m -2\cdot m\cdot(y_1 -m)$ and $- m^2 + \sum_{i=2}^P j_i^2$ of \eqref{theo:part3} are negative:}
    \begin{itemize}
    \item \jomodif{For the first term, we have}
  \begin{align}  
    m < y_1 \Longrightarrow y_1 - m \geq 1 \\\Longrightarrow 2\cdot m\cdot(y_1 -m) \geq 2\cdot m \\\Longrightarrow 2\cdot m -2\cdot m\cdot(y_1 -m) \leq 0\end{align}
    \item \jomodif{For the 2nd term, we have}
  \begin{align}  
    m^2 = (\sum_{i=2}^P j_i)\cdot(\sum_{k=2}^P j_k) 
    = \sum_{i=2}^P j_i(\sum_{k=2}^P j_k) 
    = \sum_{i=2}^P j_i(j_i + \sum_{i\neq k}^P j_k) \\
    m^2 = \sum_{i=2}^P j_i^2 + \sum_{i=2}^P \sum_{i\neq k}^P j_i\cdot j_k\\\jomodif{\mbox{As }\sum_{i=2}^P \sum_{i\neq k}^P j_i\cdot j_k \geq 0 \mbox{ we have }}m^2 \geq \sum_{i=2}^P j_i^2\\\jomodif{\mbox{which leads to }- m^2 + \sum_{i=2}^P j_i^2 \leq 0}
    \label{eq:binseq6}
\end{align}
  \end{itemize}
    So thanks to $2\cdot m -2\cdot m\cdot(y_1 -m) \leq 0$  and \jomodif{$- m^2 + \sum_{i=2}^P j_i^2 \leq 0$, we have the sum of the two previous terms, which give\begin{align}
        2\cdot m -2\cdot m\cdot(y_1 -m)- m^2 + \sum_{i=2}^P j_i^2\leq 0
    \end{align}And according to equality \eqref{theo:part3}, we obtain the inequality \eqref{theo:part1} that we wanted to prove. Then, according to that inequality, we have \begin{align}S_0 + \left(m^2 - 2\cdot m\cdot y_1 + 2\cdot m + \sum_{i=2}^P j_i^2\right) \leq S_0\end{align}} So thanks to equality \eqref{eq:binseq4} we finally have \begin{align}S \leq S_0\label{eq:binseq5}\end{align} \hfill $\square$

\end{proof}

\begin{proof}[Conjecture \eqref{eqn::upper_bound_ssquare}]
If $\gng = 0$, then $\gnv = 0$. This means that $\gsumsquares = 0 = \max(\gnv^2 + \gng - 1,0)$.

If $\gng = 1$, then the binary sequence has just one stretch of 1s. This means that $\gsumsquares = \gnv^2 = \max(\gnv^2 + \gng - 1,0)$.

If $\gng > 1$, we have a binary sequence $S_0'$ where the length of the largest stretch of 1s is $ \gnv -(\gng -1)$ and the length of the other $\gng - 1$ remaining stretches of 1s is one. So according to Theorem~\ref{the:maximisation}, \jomodif{we identify $\gnv$ as $n$, $\gng$ as $P$ and the lengths of the stretches as the integers $y_i$, which lead to} $\gsumsquares\leq (\gnv - \gng + 1)^2 + \gng - 1$. \jomodif{And as $(\gnv - \gng + 1)^2 + \gng - 1\geq 0$, we have \begin{align}
    (\gnv - \gng + 1)^2 + \gng - 1 = \max((\gnv - \gng + 1)^2 + \gng - 1,0)
\end{align}This finally leads to $\gsumsquares\leq \max((\gnv - \gng + 1)^2 + \gng - 1, 0)$.}\hfill $\square$
\end{proof}

\subsection{Conjecture~\eqref{eqn::upper_bound_ssquare2} }
We prove the selected conjecture \eqref{eqn::upper_bound_ssquare2}, which was found by the Bound Seeker.

\begin{align}
\gsumsquares \leq \left\{\begin{array}{ll}\max(\gnv^2,0)& \mbox{if }\drange = 0\wedge \min(\gnv,1) = 1 \\0& \mbox{if }\drange = 0\wedge \min(\gnv,1)=0\\\max((\gnv - 2)^2 + 2,0)& \mbox{if }\drange \geq 1 \end{array}\right.\label{eqn::upper_bound_ssquare2}
\end{align}

\begin{proof}[Conjecture \eqref{eqn::upper_bound_ssquare2}]
If $\drange = 0\wedge \min(\gnv,1) = 1$, this means that $\gng \geq 1$. \jomodif{So the binary sequence has at least a stretch of 1s. In addition, as $\drange = 0$, the binary sequence has inter-distances of 0s with the same length or no inter-distance of 0s. And  when there is no inter-distance of 0s, we have $\gng=1$ and $\gsumsquares = \gnv^2$. According to Theorem~\ref{the:maximisation}, by identifying $n = \gnv$ and $\gng = P = 1$, the maximum value of the sum of squares of lengths of stretches of 1s that we can have in a  binary sequence is $S_0=(\gnv -(\gng-1))^2 +\gng - 1 = \gnv^2$  where the number of 1s in the sequence is $\gnv$.} So that binary sequence has just one stretch of 1s of length $\gnv$. This means that $\gsumsquares = \gnv^2 = \max(\gnv^2,0)$.

If $\drange = 0\wedge \min(\gnv,1) = 1$, then $\gnv = 0$. So there are no stretches of 1s in the binary sequence. This means that $\gsumsquares = 0$.

If $\drange\geq 1$, then there are at least two \jomodif{inter-distances of 0s}. This means that there are at least three stretches of 1s. According to Theorem~\ref{the:maximisation}, we have $\gsumsquares \leq (\gnv - (\gng - 1))^2 + \gng - 1$. Note that a binary sequence is a distribution of $\gnv$ values of $1$ between $\gng$ stretches of 1s. So when the number $\gng$ of stretches of 1s  decreases, it increases the lengths of these stretches and therefore the sum of squares of lengths of stretches of 1s also increases. So $(\gnv - (\gng - 1))^2 + \gng - 1$ reaches his maximum value $(\gnv - 2)^2 + 2$ when $\gng$ reaches his minimum value $3$ and we clearly find the expression of the conjecture.  \hfill $\square$

\end{proof}

\subsection{Conjecture~\eqref{eqn::lower_bound_dsquare}}
We prove the selected conjecture \eqref{eqn::lower_bound_dsquare}, which was found by the Bound Seeker.

\begin{align}
\dsumsquares \geq \dmin^2\cdot(\gng - 1)\label{eqn::lower_bound_dsquare}
\end{align}

\begin{proof}[Conjecture \eqref{eqn::lower_bound_dsquare}]
\jomodif{ If there is no inter-distance of 0s, we have $\dmin = 0$  and  $\dsumsquares = 0 = 0^2\cdot(\gng - 1) = \dmin^2\cdot(\gng - 1)$. So the conjecture holds.

Otherwise, if there is at least one inter-distance of 0s, }let $d_i$ (with $i \in [1:\gng - 1]$) be the number of 0s of the $i$-th \jomodif{inter-distance of 0s} between stretches of~1s.
\begin{align}\dsumsquares = \sum_{i=1}^{\gng-1}d_i^2 \mbox{ and }\forall i\in [1:\gng-1], d_i \geq \dmin\\\Longrightarrow \dsumsquares \geq \dmin^2\cdot(\gng - 1)\end{align} \hfill $\square$
\end{proof}

\subsection{Conjecture~\eqref{eqn::lower_bound_dsquare2}}
We prove the selected conjecture \eqref{eqn::lower_bound_dsquare2}, which was found by the Bound Seeker.

\begin{align}
\dsumsquares \geq \left\{\begin{array}{ll}0&\mbox{ if }\gng \leq 1\\\max((\drange+1)^2 + \gng -2,0)&\mbox{ otherwise}\end{array}\right.\label{eqn::lower_bound_dsquare2}
\end{align}

\begin{proof}[Conjecture \eqref{eqn::lower_bound_dsquare2}]
If $\gng \leq 1$, then there is no \jomodif{inter-distance} of 0s between stretches of 1s. So $\dsumsquares = 0$.

If $\gng \geq 2$, then there is at least one \jomodif{inter-distance of 0s} between stretches of 1s, \jomodif{and we distinguish two cases:} \begin{itemize}
    \item \jomodif{When $\gng = 2$, there is just one inter-distance. So $\dmin = \dmax$ and $\drange = 0$. The smallest value of $\dsumsquares$ is the square of the smallest length of that inter-distance. And the smallest value of the inter-distance is $1$. Then, we have \begin{align}
        \dsumsquares \geq 1^2 = 1 = \max(1,0) =\max((0+1)^2 + 2 -2,0) \\
     \Longrightarrow \dsumsquares \geq \max((0+1)^2 + 2 -2,0) = \max((\drange+1)^2 + \gng -2,0)\\
     \Longrightarrow \dsumsquares \geq \max((\drange+1)^2 + \gng -2,0)
    \end{align}So the conjecture holds for the case $\gng = 2$.}
    \item \jomodif{When $\gng > 2$,  we have at least two inter-distances: the largest and smallest inter-distances with respective lengths $\dmax$ and $\dmin$. Note that $\dmax = \drange + \dmin$ and that there are $\gng-1$ inter-distances in the binary sequence. So we have } \begin{align}\dsumsquares = \dmax^2 + \dmin^2 + \sum_{i=1}^{\gng-3}d_i^2 = (\drange + \dmin)^2 + \dmin^2 + \sum_{i=1}^{\gng-3}d_i^2\end{align}
So, for given values of $\gng$ and $\drange$, the sum of squares $\dsumsquares$ of lengths of \jomodif{inter-distances} of 0s between stretches of 1s reaches its minimum value when \jomodif{the lengths of the inter-distances are minimum. That is }$\forall i \in [1:\gng-3]\, , \dmin = d_i = 1 $. So we have\begin{align}
    \dsumsquares \geq (\drange + 1)^2 + 1 + \sum_{i=1}^{\gng-3}1 = (\drange+1)^2 + \gng -2
\end{align} 
\jomodif{And because $(\drange+1)^2 + \gng -2\geq 0$, we have \begin{align}\dsumsquares \geq (\drange+1)^2 + \gng -2= \max((\drange+1)^2 + \gng -2,0)\end{align}So finally we have \begin{align}\dsumsquares \geq \max((\drange+1)^2 + \gng -2,0)\end{align}}
\end{itemize}
\hfill $\square$
\end{proof}

\subsection{Conjecture~\eqref{eqn::lower_bound_dsquare3}}
We prove the selected conjecture \eqref{eqn::lower_bound_dsquare3}, which was found by the Bound Seeker.

\begin{align}
\dsumsquares \geq \dmax^2\label{eqn::lower_bound_dsquare3}
\end{align}

\begin{proof}[Conjecture \eqref{eqn::lower_bound_dsquare3}]\jomodif{ If there is no inter-distance of 0s, we have $\dmin = 0$  and  $\dsumsquares = 0 = 0^2 = \dmin^2$. So the conjecture holds.

If there is at least one inter-distance of 0s, then we have}
\begin{align}
    \dsumsquares = \dmax^2 + \sum_{i=1}^{\gng-2}d_i^2
\end{align}\jomodif{As $\sum_{i=1}^{\gng-2}d_i^2\geq 0$, we have }$\dsumsquares \geq \dmax^2$. 
\hfill $\square$
\end{proof}

\subsection{Conjecture~\eqref{eqn::upper_bound_dsquare}}
We prove the selected conjecture \eqref{eqn::upper_bound_dsquare}, which was found by the Bound Seeker.

\begin{align}
\dsumsquares \leq \left\{\begin{array}{ll}0& \mbox{if }\gnv \leq 1 \\(\gn -\gnv)^2& \mbox{ otherwise} \end{array}\right.\label{eqn::upper_bound_dsquare}
\end{align}

\begin{proof}[Conjecture \eqref{eqn::upper_bound_dsquare}]
If $\gnv \leq 1$, then there is at most one stretch of 1s in the binary sequence. This means that there is no \jomodif{inter-distance of 0s} between stretches of 1s. So $\dsumsquares = 0$.

If $\gnv \geq 2$, then there are $\gn - \gnv$ 0s to distribute among \jomodif{inter-distances of 0s} between stretches of 1s. And, according to \eqref{eq:binseq6}, the distribution that gives the maximum value of $\dsumsquares$ is when the binary sequence has  just one \jomodif{inter-distance of 0s} of length $\gn-\gnv$ which is between two stretches of 1s. And because we have $\gnv \geq 2$, it is possible to build two stretches of 1s. So we can conclude that $\dsumsquares \leq (\gn -\gnv)^2$. \hfill $\square$
\end{proof}

\subsection{Conjecture~\eqref{eqn::upper_bound_dsquare2}}
We prove the selected conjecture \eqref{eqn::upper_bound_dsquare2}, which was found by the Bound Seeker.

\begin{align}
\dsumsquares \leq \left\{\begin{array}{ll}\max((\gn -\gnv -(\gng-2))^2 +\gng -2,0)& \mbox{if }\gng \geq 2 \\\max(\gng - 2,0)& \mbox{ otherwise} \end{array}\right.\label{eqn::upper_bound_dsquare2}
\end{align}

\begin{proof}[Conjecture \eqref{eqn::upper_bound_dsquare2}]
If $\gng \leq 1$, then there is no \jomodif{inter-distance of 0s} between stretches of 1s in the binary sequence. So $\dsumsquares = 0 = \max(\gng - 2,0)$.

If $\gng \geq 2$, then there are $\gn - \gnv$ values of 0 to distribute among $\gng-1$ \jomodif{inter-distances of 0s} between stretches of 1s in the binary sequence. According to Theorem~\ref{the:maximisation}, the distribution that gives the maximum value of $\dsumsquares$ is the one where the largest \jomodif{inter-distance of 0s} has a length of $\gn -\gnv -(\gng-2)$ and the remaining $\gng-2$ \jomodif{inter-distances of 0s} have a length of $1$. Which means that $\dsumsquares\leq (\gn -\gnv -(\gng-2))^2 +\gng -2$. \hfill $\square$
\end{proof}

\subsection{Conjecture~\eqref{eqn::upper_bound_smax1}}
We prove the selected conjecture \eqref{eqn::upper_bound_smax1}, which was found by the Bound Seeker.

\begin{align}
\gmax \leq \left\{\begin{array}{ll}\gn& \mbox{if }\gng = 1\wedge \dmax = 0 \\ \min(\gng,1)& \mbox{if }\gng \neq 1\wedge \dmax = 0\\
\gn - \dmax -(\gng - 2) \cdot \dmin - \gng + \min(\gng,1)& \mbox{if }\gng \neq 1\wedge \dmax \geq 1\end{array}\right.\label{eqn::upper_bound_smax1}
\end{align}

\begin{proof}[Conjecture \eqref{eqn::upper_bound_smax1}]
If \jomodif{$\gng = 1\wedge \dmax = 0$}, then there is no \jomodif{inter-distance of 0s} between stretches of 1s in the binary sequence. So the maximum value of $\gmax$ is this case is $\gn$.

If $\gng \neq 1 \wedge \dmax = 0$, then $\gng = 0$. In this case $\gmax = 0 = \max(0,1) = \max(\gng,1)$.

If $\gng \neq 1\wedge \dmax \geq 1$, then $\gng \geq 2$. Which means that $\min(\gng,1) = 1$. It also means that there is a largest \jomodif{inter-distance of 0s} of length $\dmax$, and $\gng - 2$ remaining \jomodif{inter-distances of 0s} of lengths equal, at least, to $\dmin$ which are all between stretches of 1s. Also there are $\gng - 1$ stretches of 1s of lengths at least equal to $1$, and the largest stretch of $1$ of length $\gmax$. All this leads to \begin{align}
    \gn = \gmax + \sum_{i=1}^{\gng-1}g_i+\dmax + \sum_{i=1}^{\gng-2}d_i\\
    \mbox{As }g_i \geq 1 \mbox{ and }d_i\geq \dmin, \mbox{ we have }n\geq \gmax + (\gng - 1)+\dmax + (\gng - 2) \cdot \dmin\\
    \mbox{ So }\gmax \leq\gn - \dmax -(\gng - 2) \cdot \dmin - (\gng -1)\\
    \gmax \leq\gn - \dmax -(\gng - 2) \cdot \dmin - \gng + 1\\
     \gmax \leq\gn - \dmax -(\gng - 2) \cdot \dmin - \gng + \min(\gng,1)
\end{align}  \hfill $\square$
\end{proof}

\subsection{Conjecture~\eqref{eqn::upper_bound_ssquare3}}
We prove the selected conjecture \eqref{eqn::upper_bound_ssquare3}, which was found by the Bound Seeker.

\begin{align}
\gsumsquares \leq \left\{\begin{array}{ll}\max(\gn^2,0)& \mbox{if }\gng = 1\wedge \dmax = 0 \\ \max((\min(\gng,1))^2 + \gng - 1,0)& \mbox{if }\gng \neq 1\wedge \dmax = 0\\
\max((\gn - \dmax -(\gng - 2) \cdot \dmin - \gng + 1)^2 + \gng - 1,0)& \mbox{if }\gng \neq 1\wedge \dmax \geq 1\end{array}\right.\label{eqn::upper_bound_ssquare3}
\end{align}

\begin{proof}[Conjecture \eqref{eqn::upper_bound_ssquare3}]To get the maximum value of $\gsumsquares$, we need to get the maximum value of $\gmax$.
The proof of Conjecture~\eqref{eqn::upper_bound_smax1} gives the maximum value of $\gmax$, and according to Theorem~\ref{the:maximisation}, Conjecture~\eqref{eqn::upper_bound_ssquare3} is proved.
\hfill $\square$
\end{proof}

\clearpage
\bibliographystyle{splncs04}
\bibliography{mybiblio}

\end{document}